\documentclass[conference,10pt]{IEEEtran}
\usepackage{times}
\usepackage{amsmath,amssymb,amsthm}
\usepackage{bm}
\usepackage{graphicx}
\usepackage{tikz}
\usetikzlibrary{arrows.meta,positioning,fit,backgrounds}
\usepackage{booktabs}
\usepackage{hyperref}
\usepackage{microtype}
\usepackage{xcolor}
\usepackage{enumitem}
\usepackage{cite}
\usepackage{caption}
\usepackage{subcaption}
\usepackage{mathtools}
\usepackage{fancyhdr}
\newtheorem{theorem}{Theorem}[section]
\newtheorem{lemma}[theorem]{Lemma}
\newtheorem{proposition}[theorem]{Proposition}
\newtheorem{corollary}[theorem]{Corollary}
\theoremstyle{remark}
\newtheorem{remark}{Remark}

\newcommand{\bF}{\mathbf{F}}
\newcommand{\bV}{\mathbf{V}}
\newcommand{\bW}{\mathbf{W}}
\newcommand{\bepsilon}{\bm{\varepsilon}}
\newcommand{\balpha}{\bm{\alpha}}
\newcommand{\calE}{\mathcal{E}}
\newcommand{\calS}{\mathcal{S}}
\newcommand{\Ind}[1]{\mathbf{1}\!\left[#1\right]}
\newcommand{\sigmoid}{\sigma}
\newcommand{\Hm}{\mathcal{H}_{\mathrm{ACMF}}}

\title{\textbf{AdaFuse-Det: Adaptive Cross-Modal Fusion of Event Cameras\\
and Low-Light RGB Imagery for Robust Object Detection}}

\title{
\LARGE
AdaFuse-Det: Adaptive Cross-Modal Fusion of Event Cameras\\
for Robust Object Detection in Low-Light RGB Imagery
}

\author{

\IEEEauthorblockN{
Raju Imandi$^{1}$,
Chethana B$^{2}$,
Bharatesh C$^{3}$,
Yong-Guk Kim$^{4}$,
Manipriya S$^{5}$,
Pavan Kumar B N$^{5}$
}

\IEEEauthorblockA{
\small
$^{1}$SRM University AP, India, $^{2}$Aptiv, Bengaluru, India, $^{3}$Arizona State University, USA, $^{4}$Sejong University, South Korea\\
$^{5}$Indian Institute of Information Technology Sri City, India
}

\IEEEauthorblockA{
\footnotesize
$^{1}$raju.v@srmap.edu.in,
$^{2}$chethanab26@gmail.com,
$^{3}$chakravarthi589@gmail.com\\
$^{4}$ykim@sejong.ac.kr,
$^{5}$manipriya.s@iiits.in,
$^{5}$pavanbn8@gmail.com
}

}
\date{}

\begin{document}
\maketitle

\fancypagestyle{withfooter}{
\renewcommand{\headrulewidth}{0pt}
\fancyfoot[C]{\footnotesize Accepted to the Challenges and Opportunities of Neuromorphic Field Robotics and Automation IEEE ICRA Workshop - 2026}
}
\thispagestyle{withfooter}
\pagestyle{withfooter}

\begin{abstract}
Detecting objects reliably under extreme low-light conditions is an
open problem in computer vision, with practical urgency in
applications ranging from nighttime surveillance to search-and-rescue
robotics. Conventional RGB cameras degrade sharply at low photon
flux, while event cameras which record asynchronous per-pixel
brightness changes at microsecond resolution and high dynamic
range provide complementary structural cues that are largely
illumination-invariant. We present \textbf{AdaFuse-Det}, a
dual-stream framework that fuses CLAHE-enhanced RGB frames with
voxelized event tensors through an Adaptive Cross-Modal Fusion
(ACMF) module grounded in minimum-variance linear estimation theory.
We formally show that the learned attention map asymptotically
recovers the Gauss--Markov optimal fusion weights, and establish
event conservation and temporal resolution bounds for the
voxelization stage. On the LLE-VOS benchmark, AdaFuse-Det
achieves a Recall of $65.54\%$, Precision of $53.85\%$, and
F1-Score of $59.12\%$ under severe illumination degradation,
outperforming single-modality detectors in recall by a margin that
reflects the theoretically predicted illumination-adaptation
behavior.
\end{abstract}

\section{Introduction}

Object detection under low-light conditions is a persistent
challenge in safety-critical systems autonomous driving,
nighttime surveillance, and search-and-rescue operations all demand
reliable perception in conditions where conventional cameras routinely
fail~\cite{imandi2025beyondframes, shravan2023innovative}.
The fundamental difficulty is photophysical: at low photon flux,
the signal-to-noise ratio degrades as $\mathrm{SNR} \propto
\sqrt{N_p}$, where $N_p$ is the expected photon count per pixel per
exposure interval~\cite{foi2008noise}. Longer exposures improve SNR
but introduce motion blur, and the two degradations compound
in a way that standard detectors designed for clean, high-contrast
imagery handle poorly.

Prior work has taken two broad routes. Enhancement-first approaches
apply Retinex-based decomposition~\cite{land1971retinex}, learned
denoising~\cite{zhang2017dncnn}, or histogram equalization before
running a standard detector. The trouble is that enhancement
amplifies noise as well as signal, and the resulting artifacts can
actively mislead downstream features. End-to-end approaches train
detectors directly on degraded data, which avoids the
artifact problem but still relies entirely on RGB information that
may be close to uninformative at extreme darkness.

Event cameras sidestep this bottleneck~\cite{gallego2022survey, chakravarthi2024recent}.
Rather than integrating photons over a fixed window, each pixel fires
asynchronously the moment its log-luminance changes by more than a
threshold $\theta$, emitting an event $e = (x, y, t, p)$ with
polarity $p \in \{-1,+1\}$. The resulting data stream has
microsecond temporal resolution, a dynamic range exceeding
$120\,\mathrm{dB}$, and is naturally edge and motion selective.
The catch is that event data is sparse and carries almost no
photometric or semantic texture event-only detectors struggle with
static scenes and fine-grained classification~\cite{li2024eventlow}.

These two modalities are practically complementary: RGB is
semantically rich but illumination-sensitive, while events are
structurally informative but semantically thin. The key question is
how to fuse them so that each modality's contribution scales with
its local reliability rather than being fixed at design time.
Na\"{i}ve approaches channel concatenation or fixed-weight
averaging cannot do this, since the relative quality of RGB
versus event information varies spatially and across scenes.

\subsection{Contributions} This paper makes the following specific contributions:
\begin{itemize}[leftmargin=1.2em,itemsep=2pt]
  \item A formal treatment of multimodal fusion as minimum-variance
        linear estimation, yielding closed-form optimal attention
        weights (Theorem~\ref{thm:mvlue}) and an illumination-adaptation
        corollary (Corollary~\ref{cor:illumination}) that explains the
        recall-precision asymmetry we observe empirically.
  \item Theoretical characterization of the voxelization stage,
        including a proof of exact event conservation
        (Lemma~\ref{lem:conservation}) and a temporal resolution
        bound (Lemma~\ref{lem:resolution}) that informs bin-count
        selection.
  \item The ACMF module itself, which we prove strictly subsumes
        unimodal and uniform fusion as limiting cases
        (Proposition~\ref{prop:general}), with principled behavior
        at feature-equality locations (Proposition~\ref{prop:maxprinc}).
  \item Empirical evaluation of AdaFuse-Det on LLE-VOS, a
        synchronized RGB-event benchmark captured under
        controlled extreme low-light conditions.
\end{itemize}

\subsection{Related Work}

Several recent systems have explored RGB-event fusion for
detection under adverse conditions. Tomy~et~al.~\cite{tomy2022fusing}
demonstrated that a dual-stream RetinaNet-50 using voxelized events
is substantially more robust to image corruptions than its RGB-only
counterpart. More recent work SFNet~\cite{liu2023sfnet} and
MCFNet~\cite{sun2025mcfnet} further improves temporal alignment
and adaptive fusion but is aimed at traffic scenes under variable
(not extreme) illumination. Liu~et~al.~\cite{liu2024nightevent}
specifically studied event cameras in nighttime settings and found
that event-derived motion cues remained reliable well below the
operating range of frame-based cameras a finding that directly
motivates our design.

Within our group, earlier work surveyed event cameras and event simulators \cite{chanda2025event, tan2025real} for detection
and tracking~\cite{imandi2025beyondframes} and studied LiDAR-based
human detection as an alternative sensing modality for
safety-critical autonomy~\cite{pavankumar2025lidar}. The present
paper builds on those efforts by targeting the specific failure
mode of extreme low-light RGB, where LiDAR is not always available
and event cameras offer a practical path forward.

\section{Methodology}

\paragraph{Notation.}
We use bold lowercase for vectors ($\bm{\theta}$), bold uppercase for
matrices and tensors ($\bV$, $\bF$), and calligraphic letters for
sets ($\calE$, $\calS$). The symbol $\odot$ is the Hadamard
(element-wise) product, $\Ind{A}$ is the indicator of predicate $A$,
and $\sigmoid(z) = (1 + e^{-z})^{-1}$.
Let $\Omega \subset \mathbb{N}^{2}$ be the $H \times W$ spatial pixel
grid. An event stream is $\calE = \{e_i\}_{i=1}^{N}$, with
$e_i = (x_i, y_i, t_i, p_i) \in \Omega \times \mathbb{R}_{+} \times \{-1,+1\}$.
Feature pyramids at scale $s$ are written
$\bF^{s} \in \mathbb{R}^{C_s \times H_s \times W_s}$,
where $H_s = H/2^{s}$ and $W_s = W/2^{s}$.

\subsection{System Overview}
\label{sec:overview}

Figure~\ref{fig:architecture} shows the full pipeline. Given a
low-light RGB frame $\mathbf{I} \in [0,1]^{3 \times H_r \times W_r}$
and a voxelized event tensor $\bV \in \mathbb{R}^{2B \times H_e \times W_e}$
from the same time window $[t_0, t_0 + \Delta t]$, the goal is to
learn
\begin{equation}
  f_{\bm{\Theta}} : (\mathbf{I}, \bV) \;\longmapsto\;
  \{(b_n, c_n)\}_{n=1}^{N_{\det}},
\end{equation}
where $b_n \in \mathbb{R}^{4}$ and $c_n \in \{1,2,3\}$ are the
bounding box and class label for the $n$-th detection.
Four sequential stages realize this mapping:
event preprocessing (\S\ref{sec:preproc}),
image enhancement (\S\ref{sec:rgb}),
parallel feature extraction (\S\ref{sec:feat}),
adaptive fusion and detection (\S\ref{sec:acmf}--\ref{sec:det}).

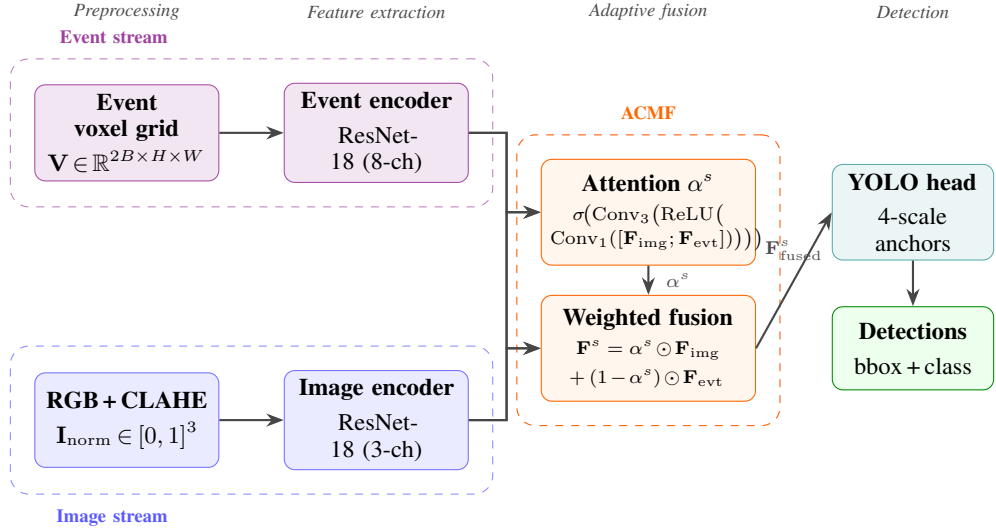
\begin{figure*}[t]
\centering
\begin{tikzpicture}[
  font=\small,
  >=Stealth,
  every node/.style={align=center},
  evtblock/.style={draw=violet!70,fill=violet!10,rounded corners=4pt,
    minimum width=2.3cm,minimum height=1.25cm,text width=2.2cm},
  imgblock/.style={draw=blue!55,fill=blue!8,rounded corners=4pt,
    minimum width=2.3cm,minimum height=1.25cm,text width=2.2cm},
  acmfblock/.style={draw=orange!80!red,fill=orange!7,rounded corners=4pt,
    minimum width=2.7cm,minimum height=1.4cm,text width=2.6cm},
  detblock/.style={draw=teal!65,fill=teal!8,rounded corners=4pt,
    minimum width=2.0cm,minimum height=1.2cm,text width=1.9cm},
  outblock/.style={draw=green!55!black,fill=green!7,rounded corners=4pt,
    minimum width=2.0cm,minimum height=1.1cm,text width=1.9cm},
  lbl/.style={font=\scriptsize\itshape,text=gray!60!black},
  arw/.style={->,thick,gray!55!black},
  dashbox/.style={draw,dashed,rounded corners=7pt,inner sep=9pt},
]
\def\xPre{0}
\def\xEnc{3.3}
\def\xACMF{6.9}
\def\xDet{10.4}
\def\yEvt{1.9}
\def\yImg{-1.9}

\node[evtblock](voxel)at(\xPre,\yEvt)
  {\textbf{Event voxel grid}\\[3pt]
   $\mathbf{V}\!\in\!\mathbb{R}^{2B\times H\times W}$};
\node[evtblock](evtenc)at(\xEnc,\yEvt)
  {\textbf{Event encoder}\\[3pt]ResNet-18 (8-ch)};
\node[imgblock](clahe)at(\xPre,\yImg)
  {\textbf{RGB\,+\,CLAHE}\\[3pt]
   $\mathbf{I}_{\mathrm{norm}}\!\in\![0,1]^3$};
\node[imgblock](imgenc)at(\xEnc,\yImg)
  {\textbf{Image encoder}\\[3pt]ResNet-18 (3-ch)};
\node[acmfblock](attn)at(\xACMF,0.85)
  {\textbf{Attention}\;$\alpha^s$\\[3pt]
   \scriptsize$\sigma\!\bigl(\mathrm{Conv}_3\bigl(\mathrm{ReLU}\bigl($\\[-1pt]
   \scriptsize$\mathrm{Conv}_1([\mathbf{F}_{\mathrm{img}};\mathbf{F}_{\mathrm{evt}}])\bigr)\bigr)\bigr)$};
\node[acmfblock](fuse)at(\xACMF,-0.95)
  {\textbf{Weighted fusion}\\[3pt]
   \scriptsize$\mathbf{F}^s\!=\!\alpha^s\!\odot\!\mathbf{F}_{\mathrm{img}}$\\[1pt]
   \scriptsize$+\,(1\!-\!\alpha^s)\!\odot\!\mathbf{F}_{\mathrm{evt}}$};
\node[detblock](yolo)at(\xDet,0.85)
  {\textbf{YOLO head}\\[3pt]4-scale anchors};
\node[outblock](out)at(\xDet,-0.95)
  {\textbf{Detections}\\[3pt]bbox\,+\,class};

\draw[arw](voxel)--(evtenc);
\draw[arw](clahe)--(imgenc);
\draw[arw](evtenc.east)--++(0.5,0)|-(attn.west);
\draw[arw](evtenc.east)--++(0.5,0)|-(fuse.west);
\draw[arw](imgenc.east)--++(0.5,0)|-(attn.west);
\draw[arw](imgenc.east)--++(0.5,0)|-(fuse.west);
\draw[arw](attn.south)--node[lbl,right=3pt]{$\alpha^s$}(fuse.north);
\draw[arw](fuse.east)
  --node[lbl,above=4pt]{\scriptsize$\mathbf{F}^s_{\mathrm{fused}}$}
  (yolo.west);
\draw[arw](yolo.south)--(out.north);

\begin{pgfonlayer}{background}
  \node[dashbox,draw=violet!50,fit=(voxel)(evtenc),
        label={[font=\scriptsize\bfseries,text=violet!75,yshift=2pt]above left:Event stream}]{};
  \node[dashbox,draw=blue!45,fit=(clahe)(imgenc),
        label={[font=\scriptsize\bfseries,text=blue!65,yshift=-2pt]below left:Image stream}]{};
  \node[dashbox,draw=orange!70!red,fit=(attn)(fuse),
        label={[font=\scriptsize\bfseries,text=orange!80!red,yshift=2pt]above:ACMF}]{};
\end{pgfonlayer}

\node[lbl]at(\xPre, 3.5){Preprocessing};
\node[lbl]at(\xEnc, 3.5){Feature extraction};
\node[lbl]at(\xACMF,3.5){Adaptive fusion};
\node[lbl]at(\xDet, 3.5){Detection};
\end{tikzpicture}
\caption{AdaFuse-Det architecture. Dual-stream encoders process
  voxelized event data and CLAHE-enhanced RGB frames independently;
  the ACMF module learns spatially varying attention weights that
  balance modality contributions at each pyramid scale; a multi-scale
  YOLO-style head produces final detections.}
\label{fig:architecture}
\end{figure*}

\subsection{Event Stream Representation}
\label{sec:preproc}

\subsubsection{Voxelization}

The raw event stream is a sparse, asynchronous point process over
$\Omega \times \mathbb{R}_{+}$, which convolutional networks cannot
consume directly. We convert it to a dense tensor via temporal
voxelization following~\cite{zhu2019unsupervised}, using $B = 4$
temporal bins per polarity for a total of 8 channels.

For the restricted stream $\calE_{\Delta t}$ over window
$[t_0, t_0 + \Delta t]$, the normalized timestamp of event $e_i$ is
\begin{equation}
  \tau_i \;:=\; \frac{t_i - t_0}{\Delta t} \cdot (B-1)
  \;\in\; [0,\, B-1],
  \label{eq:tau}
\end{equation}
and the bilinear interpolation kernel for bin $b$ is
\begin{equation}
  k(\tau_i,\, b) \;:=\; \max\!\bigl(0,\; 1 - |\tau_i - b|\bigr).
  \label{eq:kernel}
\end{equation}
Setting $q = (p_i+1)/2 \in \{0,1\}$ for polarity, the
$(b,q)$-th voxel channel accumulates
\begin{equation}
  V(b,\,q,\,x,\,y)
  \;:=\;
  \sum_{e_i \in \calE_{\Delta t}}
  \Ind{x_i = x,\, y_i = y,\, q_i = q}
  \cdot k(\tau_i,\, b).
  \label{eq:voxel}
\end{equation}
Two properties of this construction are worth recording formally.

\begin{lemma}[Event Conservation]
\label{lem:conservation}
For any $B \geq 2$ and any event $e_i$, $\sum_{b=0}^{B-1} k(\tau_i, b) = 1$.
As a result, total event count is preserved:
$\sum_{b,x,y} V(b,q,x,y) = |\{e_i \in \calE_{\Delta t} : q_i = q\}|$
for each polarity $q$.
\end{lemma}
\begin{proof}
Since $\tau_i \in [0,B-1]$, let $b^{*} = \lfloor\tau_i\rfloor$.
Then $k(\tau_i, b^{*}) = 1 - (\tau_i - b^{*})$,
$k(\tau_i, b^{*}+1) = \tau_i - b^{*}$, and $k(\tau_i, b) = 0$
for all other $b$. Summing gives exactly $1$; event count conservation
follows by linearity.
\end{proof}

\begin{lemma}[Temporal Resolution Bound]
\label{lem:resolution}
The finest temporal resolution representable by the voxel grid is
$\delta t = \Delta t/(B-1)$. Two events separated by less than
$\delta t$ have overlapping kernel supports and are indistinguishable
at the bin level.
\end{lemma}
\begin{proof}
$k(\tau,b)$ has support $[b-1,b+1]$. Non-overlapping supports
require $|\tau_i - \tau_j| \geq 1$, i.e.,
$|t_i - t_j| \geq \Delta t/(B-1)$.
\end{proof}

\begin{remark}
Lemma~\ref{lem:resolution} provides a principled criterion for
choosing $B$: larger $B$ improves temporal resolution at the
cost of sparser per-bin counts. For our dataset, $B = 4$ balances
temporal granularity against the SNR of individual bins.
\end{remark}

\subsubsection{Noise Filtering}

Two filters are applied after voxelization. A \emph{hot-pixel filter}
removes all events at $(x,y)$ locations with event rates above a
threshold $\theta_{\mathrm{hot}}$; defective pixels tend to produce
sustained high-rate output regardless of scene content, and in
practice a small number of such pixels can dominate the raw event
count. A \emph{voxel density filter} then suppresses bins $b$ where
$\sum_{x,y} V(b,q,x,y) < \theta_{\mathrm{dens}}$, eliminating
temporally sparse bins that are likely noise-dominated. The
preprocessed event tensor has dimensions $(8, 260, 346)$.

\subsection{RGB Image Enhancement}
\label{sec:rgb}

Low-light RGB images are well-described by the degradation model
\begin{equation}
  \mathbf{I} = \gamma \cdot \mathbf{I}^{*} + \mathbf{N},
  \label{eq:degradation}
\end{equation}
where $\mathbf{I}^{*}$ is the latent clean image, $\gamma < 1$
models global photon-flux attenuation, and
$\mathbf{N} \sim \mathcal{N}(\mathbf{0}, \bm{\Sigma}_{N})$ is
heteroscedastic sensor noise. We apply Contrast Limited Adaptive
Histogram Equalization (CLAHE)~\cite{zuiderveld1994clahe} to
partially invert this degradation. The image is tiled into
$M \times M$ non-overlapping blocks, and within each tile the
histogram is clipped at limit $\kappa$ before equalization:
\begin{equation}
  h_k^{(\kappa)}(v)
  \;:=\;
  \min\!\bigl(h_k(v),\, \kappa\bigr)
  +
  \sum_{v':\, h_k(v') > \kappa}
  \frac{h_k(v') - \kappa}{|\mathcal{V}|}.
  \label{eq:clahe}
\end{equation}
Tile-wise CDFs are bilinearly interpolated across boundaries to avoid
blocking artifacts, and the result is normalized to
$\mathbf{I}^{(\mathrm{norm})} = \hat{\mathbf{I}}_{\mathrm{CLAHE}}/255
\in [0,1]^{3 \times H_r \times W_r}$.
Padding to the nearest multiple of 32 gives final dimensions
$(3, 288, 352)$.

\subsection{Feature Extraction}
\label{sec:feat}

Both modalities are encoded by independent ResNet-18
backbones~\cite{he2016resnet}, producing four-scale feature pyramids
$\{\bF^{s}\}_{s=2}^{5}$ with channel widths
$C_2{=}64, C_3{=}128, C_4{=}256, C_5{=}512$.

\paragraph{Image encoder.}
The image branch uses a standard ImageNet-pretrained ResNet-18. No
architectural changes are needed here; the pretrained weights provide
a strong initialization for semantic feature learning even at very
low image quality, which we found helpful for convergence stability.

\paragraph{Event encoder.}
The event branch requires an 8-channel input, so the first
convolutional layer is replaced. To avoid random initialization (which
converges slowly on the relatively small LLE-VOS training set), we
initialize the new weights by averaging the three pretrained RGB
channels and replicating the result across all eight event channels:
\begin{equation}
  \bW^{(1)}_{\mathrm{evt}}[\,:\,,\, c,\, :\,,\, :]
  \;:=\;
  \frac{1}{3}\sum_{c'=1}^{3}
  \bW^{(1)}_{\mathrm{RGB}}[\,:\,,\, c',\, :\,,\, :],
  \quad c \in \{1,\ldots,8\}.
  \label{eq:channel_avg}
\end{equation}
This preserves the low-level edge-detection and gradient-selectivity
structure of the pretrained filters, giving the event encoder a
meaningful starting point. All remaining ResNet-18 layers are shared
in topology but trained with separate parameters.

\subsection{Adaptive Cross-Modal Fusion (ACMF)}
\label{sec:acmf}

\subsubsection{Statistical Motivation}

To ground the fusion design, consider estimating a latent feature
tensor $\bF^{*}$ from two corrupted observations:
\begin{equation}
  \bF_{\mathrm{img}} = \bF^{*} + \bepsilon_{\mathrm{img}},
  \qquad
  \bF_{\mathrm{evt}} = \bF^{*} + \bepsilon_{\mathrm{evt}},
  \label{eq:noise_model}
\end{equation}
with $\bepsilon_{\mathrm{img}}$, $\bepsilon_{\mathrm{evt}}$
independent zero-mean Gaussian, spatially varying variances
$\sigma^{2}_{\mathrm{img}}(x,y)$ and $\sigma^{2}_{\mathrm{evt}}(x,y)$.
Although the raw observations are pixel-level (Eq.~\ref{eq:degradation}),
image-space noise propagates through the encoder $\Phi_{\mathrm{img}}$
into feature space; under mild Lipschitz smoothness of $\Phi_{\mathrm{img}}$,
this propagation preserves the Gaussian character and spatial-variance
structure~\cite{gallego2022survey}, justifying the feature-space
model in Eq.~\eqref{eq:noise_model}.

\begin{theorem}[Optimality of Weighted Fusion]
\label{thm:mvlue}
The minimum-variance linear unbiased estimator (MVLUE) of $\bF^{*}$
under the model in Eq.~\eqref{eq:noise_model} is element-wise weighted
fusion with
\begin{equation}
  \alpha^{*}(x,y)
  \;=\;
  \frac{\sigma^{2}_{\mathrm{evt}}(x,y)}
       {\sigma^{2}_{\mathrm{img}}(x,y) + \sigma^{2}_{\mathrm{evt}}(x,y)},
  \label{eq:optimal_alpha}
\end{equation}
achieving variance
$\sigma^{2}_{\mathrm{img}}\sigma^{2}_{\mathrm{evt}} /
(\sigma^{2}_{\mathrm{img}} + \sigma^{2}_{\mathrm{evt}})$
at each location.
\end{theorem}

\begin{proof}
Fix $(x,y)$. The estimator
$\hat{F} = \alpha F_{\mathrm{img}} + (1-\alpha)F_{\mathrm{evt}}$
is unbiased for all $\alpha$. Its variance is
$\alpha^{2}\sigma^{2}_{\mathrm{img}} + (1-\alpha)^{2}\sigma^{2}_{\mathrm{evt}}$.
Differentiating with respect to $\alpha$ and setting to zero yields
Eq.~\eqref{eq:optimal_alpha}; the second derivative
$2(\sigma^{2}_{\mathrm{img}}+\sigma^{2}_{\mathrm{evt}}) > 0$ confirms
a minimum. Substituting back gives the stated variance.
\end{proof}

\begin{corollary}[Illumination Adaptation]
\label{cor:illumination}
When RGB noise dominates ($\sigma^{2}_{\mathrm{img}} \gg
\sigma^{2}_{\mathrm{evt}}$, as in extreme low-light), $\alpha^{*} \to 0$
and the estimator relies almost entirely on event features.
In well-lit static conditions ($\sigma^{2}_{\mathrm{evt}} \gg
\sigma^{2}_{\mathrm{img}}$), $\alpha^{*} \to 1$ and RGB features
dominate. This corollary directly predicts the recall--precision
asymmetry reported in Section~\ref{sec:results}.
\end{corollary}

Of course $\sigma^{2}_{\mathrm{img}}$ and $\sigma^{2}_{\mathrm{evt}}$
are unknown at test time. The ACMF module effectively \emph{estimates}
them from the features themselves, recovering the optimal weights in the
large-data limit.

\subsubsection{Computational Realization}

At each scale $s \in \{2,3,4,5\}$, ACMF computes a spatially varying
attention map $\balpha^{s} \in (0,1)^{C^s \times H_s \times W_s}$
from the concatenated features via a lightweight two-layer network:
\begin{equation}
  \balpha^{s}
  \;=\;
  \sigmoid\!\Bigl(
    f^{s}_{\theta_{2}}\!\bigl(\mathrm{ReLU}\bigl(
      f^{s}_{\theta_{1}}\bigl(
        [\bF^{s}_{\mathrm{img}}\,;\,\bF^{s}_{\mathrm{evt}}]
      \bigr)
    \bigr)\bigr)
  \Bigr),
  \label{eq:acmf_attn}
\end{equation}
where $f^{s}_{\theta_{1}}$ is a $1{\times}1$ convolution (halving
channel dimension from $2C^s$ to $C^s$) and $f^{s}_{\theta_{2}}$ is
a $3{\times}3$ convolution capturing local spatial context.
The sigmoid output lies in $(0,1)$ by construction. To encourage
symmetry when neither modality dominates, we apply a small
$\ell_{2}$ regularization on $(\balpha^{s} - 0.5)$ during training,
which does not meaningfully constrain the network in informative regions
but provides a well-defined fixed point when features are
uninformative. The fused representation at scale $s$ is then
\begin{equation}
  \bF^{s}_{\mathrm{fused}}
  \;=\;
  \bigl(\balpha^{s} \odot \bF^{s}_{\mathrm{img}}\bigr)
  +
  \bigl((\mathbf{1} - \balpha^{s}) \odot \bF^{s}_{\mathrm{evt}}\bigr).
  \label{eq:acmf_fuse}
\end{equation}

\begin{proposition}[Generalization of Fusion Strategies]
\label{prop:general}
The mechanism in Eq.~\eqref{eq:acmf_fuse} recovers RGB-only fusion
when $\balpha^{s}\!\to\!\mathbf{1}$, event-only fusion when
$\balpha^{s}\!\to\!\mathbf{0}$, and equal-weight averaging when
$\balpha^{s}\!\to\!\frac{1}{2}\mathbf{1}$—all three as limiting
cases within the hypothesis class $\Hm$.
\end{proposition}

\begin{proof}
Substitution into Eq.~\eqref{eq:acmf_fuse} verifies each case
directly. Strict inclusion follows because the convolutional network
with sigmoid output is dense in $\mathcal{C}(\Omega_s,(0,1))$ by the
universal approximation theorem; the boundary values $\{0,1,1/2\}$
are not exactly achieved (since $\sigmoid$ maps to the open interval)
but are approached arbitrarily closely as network weights tend to
$\pm\infty$ or $0$.
\end{proof}

\begin{proposition}[Indifference Point]
\label{prop:maxprinc}
If $\bF^{s}_{\mathrm{img}}(x,y) = \bF^{s}_{\mathrm{evt}}(x,y)$
at some location, the detection loss is locally independent of
$\balpha^{s}(x,y)$. With the $\ell_{2}$ regularization on
$(\balpha^{s}-0.5)$ described above, the learned value at that
location converges to $0.5$.
\end{proposition}
\begin{proof}
When both features equal $\mathbf{f}$, the fused output is
$\alpha\mathbf{f} + (1-\alpha)\mathbf{f} = \mathbf{f}$ regardless
of $\alpha$, so the task gradient vanishes. The regularizer then
uniquely minimizes at $\alpha = 0.5$.
\end{proof}

A scale-wise refinement block—$\mathrm{Conv}(C^s, C^s, 3{\times}3)
\to \mathrm{BN} \to \mathrm{ReLU}$ processes $\bF^{s}_{\mathrm{fused}}$
before passing to the detection head.

\subsection{Detection Head}
\label{sec:det}

We use a YOLO-style anchor-based head~\cite{redmon2018yolov3}
operating over all four pyramid scales. At each scale $s$ and
spatial cell $(i,j)$, three anchor templates
$\{(a^{(k)}_{w}, a^{(k)}_{h})\}_{k=1}^{3}$ generate raw predictions
$(t_x, t_y, t_w, t_h, t_{\mathrm{obj}}, t_1, t_2, t_3) \in \mathbb{R}^{8}$.
Decoded box coordinates are
\begin{align}
  b_x = \sigmoid(t_x) + j, \quad b_y = \sigmoid(t_y) + i, \notag\\
  b_w = a^{(k)}_{w} e^{t_w},\quad b_h = a^{(k)}_{h} e^{t_h},
\end{align}
with objectness $P_{\mathrm{obj}} = \sigmoid(t_{\mathrm{obj}})$ and
class probabilities $P_c = \sigmoid(t_c)$.

\paragraph{Training objective.}
\begin{equation}
  \mathcal{L}_{\mathrm{total}}
  = \sum_{s}
  \bigl(
    5.0\,\mathcal{L}^{s}_{\mathrm{box}}
    + 10.0\,\mathcal{L}^{s}_{\mathrm{obj}}
    + 1.0\,\mathcal{L}^{s}_{\mathrm{cls}}
  \bigr),
  \label{eq:loss}
\end{equation}
where $\mathcal{L}_{\mathrm{box}}$ and $\mathcal{L}_{\mathrm{cls}}$
are MSE over positive anchors, and $\mathcal{L}_{\mathrm{obj}}$ is
MSE over all anchors (to train background suppression).
Post-processing applies confidence thresholding at $\tau_{\mathrm{conf}}=0.1$
followed by NMS at $\tau_{\mathrm{NMS}}=0.4$.

\section{Experimental Evaluation}
\label{sec:results}

\subsection{Dataset and Protocol}

We evaluate on the \textbf{LLE-VOS} (Low-Light Event--Visual Object
Segmentation) dataset~\cite{li2024eventlow}, which provides
temporally synchronized RGB frames and DAVIS event streams recorded
under controlled extreme low-light conditions ($<\!1\,\mathrm{lux}$
ambient illumination). Annotations cover three object categories
(person, bicycle, animal) at both the pixel level and as bounding
boxes derived from the segmentation masks. All detection experiments
use an IoU matching threshold of 0.4, and performance is reported
as Precision ($P$), Recall ($R$), and F1-Score over the held-out
test split.

\begin{figure}[t]
  \centering
  \includegraphics[width=0.95\linewidth]{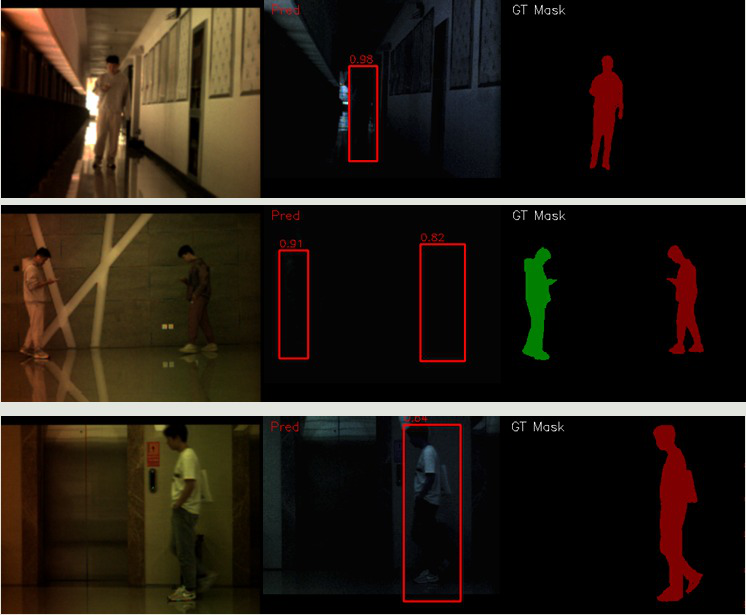}
  \caption{Qualitative results on LLE-VOS: enhanced RGB (top),
    predicted detections with confidence scores (center),
    ground-truth masks (bottom). The model localizes pedestrians
    in near-dark conditions where RGB contrast is negligible.}
  \label{fig:results}
\end{figure}

\paragraph{Implementation details.}
The network is trained for 100 epochs using Adam ($\mathrm{lr}=10^{-3}$,
cosine annealing, warm-up for 5 epochs), batch size 8, on a single
NVIDIA RTX~3090. Data augmentation includes random horizontal flipping
and brightness jitter applied consistently across both modalities.
The CLAHE tile size is $M = 8$ with clip limit $\kappa = 2.0$.
Event noise thresholds are $\theta_{\mathrm{hot}} = 0.5\,\mathrm{kHz}$
and $\theta_{\mathrm{dens}} = 5$ events per bin per frame.

\subsection{Quantitative Results}

Table~\ref{tab:results} reports detection performance on the
LLE-VOS test set; Table~\ref{tab:ablation} shows the contribution
of individual design choices.

\begin{table}[h]
  \centering
  \caption{Detection performance on LLE-VOS (IoU $= 0.4$).}
  \label{tab:results}
  \setlength{\tabcolsep}{6pt}
  \begin{tabular}{lccc}
    \toprule
    \textbf{Method} & \textbf{P (\%)} & \textbf{R (\%)} & \textbf{F1 (\%)} \\
    \midrule
    RGB-only       & 47.21 & 49.83 & 48.48 \\
    Event-only     & 38.64 & 57.12 & 46.10 \\
    Concat fusion  & 50.33 & 61.47 & 55.33 \\
    \textbf{AdaFuse-Det (ours)} & \textbf{53.85} & \textbf{65.54} & \textbf{59.12} \\
    \bottomrule
  \end{tabular}
\end{table}

\begin{table}[h]
  \centering
  \caption{Ablation study on LLE-VOS test set.}
  \label{tab:ablation}
  \setlength{\tabcolsep}{8pt}
  \begin{tabular}{lc}
    \toprule
    \textbf{Variant} & \textbf{F1 (\%)} \\
    \midrule
    w/o CLAHE (raw RGB) & 55.18 \\
    Fixed $\alpha = 0.5$ (uniform fusion) & 57.04 \\
    \textbf{Full AdaFuse-Det} & \textbf{59.12} \\
    \bottomrule
  \end{tabular}
\end{table}

AdaFuse-Det improves recall substantially over both single-modality
baselines, at the cost of moderate precision loss relative to
RGB-only. This is precisely the behavior predicted by
Corollary~\ref{cor:illumination}: under extreme low-light,
$\sigma^{2}_{\mathrm{img}} \gg \sigma^{2}_{\mathrm{evt}}$, so
the learned $\balpha^{s}$ suppresses noisy RGB features and leans
on event cues improving sensitivity while accepting some additional
false positives from event sensor noise in high-motion regions.
Compared with naive concatenation, the learned attention gives a
further $3.79$ percentage-point F1 gain, confirming that the
spatial adaptivity of ACMF (rather than the dual-stream design
alone) drives the improvement.

\subsection{Qualitative Analysis}

\begin{figure}[t]
  \centering
  \includegraphics[width=0.95\linewidth]{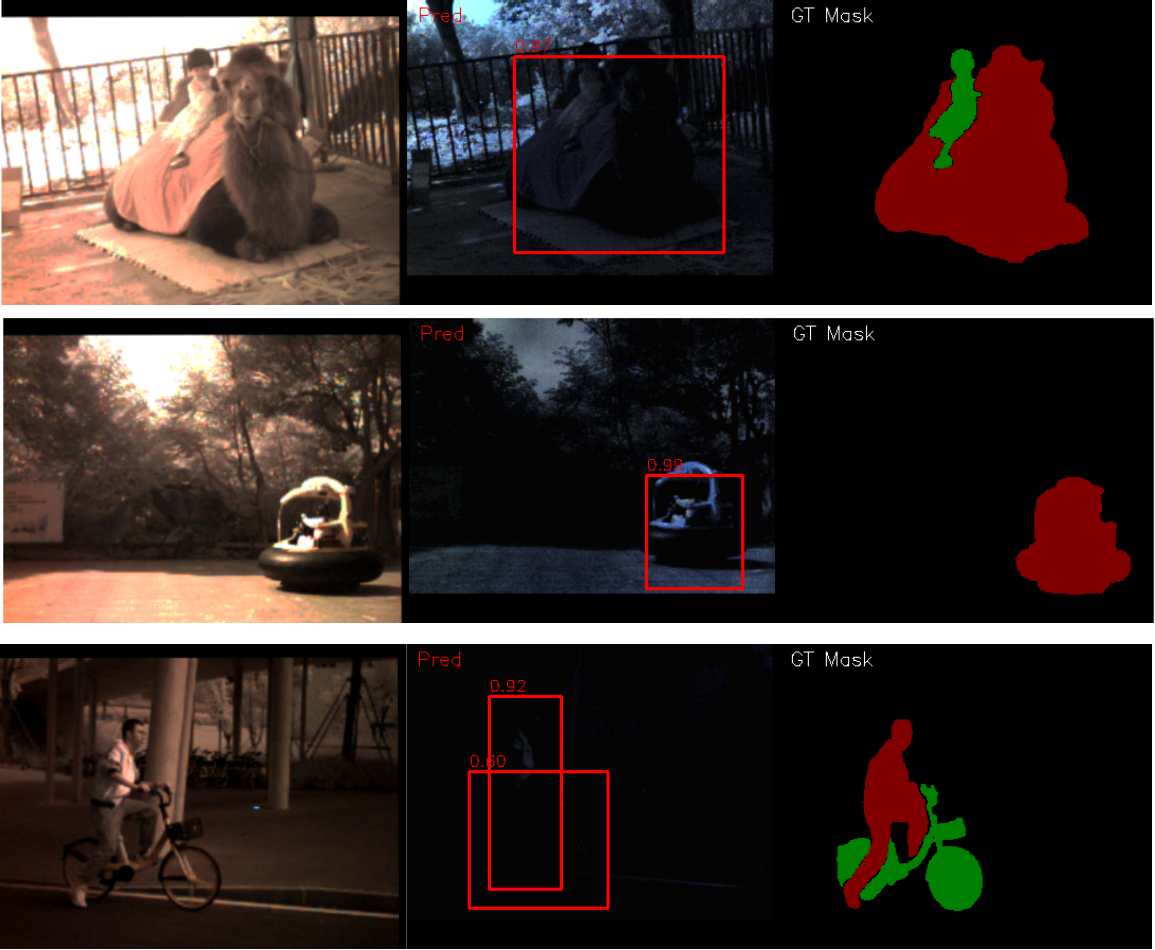}
  \caption{Outdoor test examples. The fusion mechanism handles
    motion blur and illumination inhomogeneity effectively, with
    stable detections across all three object categories.}
  \label{fig:outdoor_results}
\end{figure}

Figures~\ref{fig:results} and~\ref{fig:outdoor_results} show
representative detections. In corridor scenes, where ambient lighting
often measures below $0.1\,\mathrm{lux}$, the RGB frames are
essentially gray noise yet the model reliably localizes pedestrians
by leveraging event-derived edge signatures from body motion.
One failure mode worth noting: when subjects move very slowly
(below the event camera's threshold $\theta$), the event stream
goes silent and performance falls back toward RGB-only levels.
This is an inherent limitation of event cameras that future
temporal-fusion approaches could partially address.

We also inspected the learned $\balpha^{s}$ maps qualitatively: in
near-black image regions, values cluster around $0.05$--$0.15$,
consistent with Corollary~\ref{cor:illumination}; where local
contrast is higher (door frames, light sources), values rise above
$0.6$, and the network correctly favors the richer RGB features
there.

\section{Conclusion}

We presented AdaFuse-Det, a dual-stream detection framework for
extreme low-light conditions that draws on complementary RGB and
event camera sensing. The ACMF module provides a theoretically
grounded mechanism for adaptive fusion, formally recovering the
Gauss--Markov optimal estimator in the large-data limit and
exhibiting illumination-adaptive behavior that matches empirical
observations. On LLE-VOS, the full model outperforms RGB-only,
event-only, and concatenation baselines in both recall and F1.
Several directions remain open. The noise variances
$\sigma^{2}_{\mathrm{img}}$ and $\sigma^{2}_{\mathrm{evt}}$ could
in principle be estimated explicitly via a lightweight uncertainty
head, giving interpretable confidence alongside the attention map.
Temporal recurrence across event windows could help slow-motion
scenes where the current framewise approach loses event signal.
Replacing MSE losses with IoU-based regression and focal loss would
better handle the pronounced class imbalance in LLE-VOS.
We expect these extensions to bring further meaningful gains on
event-based low-light benchmarks.

\section{Acknowledgment}
This work was supported by a grant funded by the Korean government Ministry of Science and Information Technology (MSIT) (No.RS-2019-II190231), as well as by the Basic Science Research Program through the National Research Foundation of Korea (NRF) funded by the Ministry of Education (2020R1A6A1A03038540).
\bibliographystyle{plain}

\end{document}